\documentclass[a4paper]{article}

\usepackage{hyperref}
\usepackage{url}
\usepackage{fancyvrb}
\usepackage{amsmath}
\usepackage{amssymb}
\usepackage{amsthm}
\usepackage[ruled,vlined]{algorithm2e}
\usepackage{graphicx}
\usepackage{svg}
\usepackage{verbatim}
\usepackage{natbib}
\usepackage[inline]{enumitem}
\usepackage{subcaption}
\usepackage{booktabs}
\usepackage{multirow}
\usepackage{hhline}

\graphicspath{ {./images/} }

\theoremstyle{plain} 
\newtheorem{proposition}{Proposition}
\newtheorem{cor}[proposition]{Corollary}
\newtheorem{lemma}[proposition]{Lemma}

\title{Attention-Based Clustering:\\ Learning a Kernel from Context}
\date{}
\author{Samuel Coward, Erik Visse-Martindale, Chithrupa Ramesh\\Zuken Ltd\thanks{
Zuken Limited, 1500 Aztec West, Almondsbury, Bristol, BS32 4RF, United Kingdom\newline
\url{https://www.zuken.com/en/}\newline
\texttt{\{sam.coward,erik.visse-martindale,chithrupa.ramesh\}@uk.zuken.com}
}}

\begin{document}
\VerbatimFootnotes

\maketitle

\begin{abstract}
In machine learning, no data point stands alone. We believe that context is an underappreciated concept in many machine learning methods. We propose Attention-Based Clustering (ABC), a neural architecture based on the attention mechanism, which is designed to learn latent representations that adapt to context within an input set, and which is inherently agnostic to input sizes and number of clusters. By learning a similarity kernel, our method directly combines with any out-of-the-box kernel-based clustering approach. We present competitive results for clustering Omniglot characters and include analytical evidence of the effectiveness of an attention-based approach for clustering. 
\end{abstract}

\section{Introduction}

Many problems in machine learning involve modelling the relations between elements of a set. A notable example, and the focus of this paper, is clustering, in which the elements are grouped according to some shared properties. A common approach uses kernel methods: a class of algorithms that operate on pairwise similarities, which are obtained by evaluating a specific kernel function~\citep{FCMR2008kernelMethodsSurvey}. However, for data points that are not trivially comparable, specifying the kernel function is not straightforward.

With the advent of deep learning, this gave rise to metric learning frameworks where a parameterized binary operator, either explicitly or implicitly, is taught from examples how to measure the distance between two points~\citep{koch2015siamese, ZK2015siameseLike, hsu2018learning, wojke2018deep, hsu2019multiclass}. These cases operate on the assumption that there exists a global metric, that is, the distance between points depends solely on the two operands. This assumption disregards situations where the underlying metric is contextual, by which we mean that the distance between two data points may depend on some structure of the entire dataset.

We hypothesize that the context provided by a set of data points can be helpful in measuring the distance between any two data points in the set. As an example of where context might help, consider the task of clustering characters that belong to the same language. There are languages, like Latin and Greek, that share certain characters, for example the Latin T and the Greek upper case $\tau$.\footnote{To the extend that there is not even a LaTeX command \Verb|\Tau|} However, given two sentences, one from the Aeneid and one from the Odyssey, we should have less trouble clustering the same character in both languages correctly due to the context, even when ignoring any structure or meaning derived from the sentences themselves. Indeed, a human performing this task will not need to rely on prior knowledge of the stories of Aeneas or Odysseus, nor on literacy in Latin or Ancient Greek. As a larger principle, it is well recognized that humans perceive emergent properties in configurations of objects, as documented in the Gestalt Laws of Perceptual Organization~\citep[Chapter 2]{palmer1999vision}. 

We introduce Attention-Based Clustering (ABC) which uses context to output pairwise similarities between the data points in the input set. Our model is trained with ground-truth labels and can be used with an unsupervised clustering method to obtain cluster labels. To demonstrate the benefit of using ABC over pairwise metric learning methods, we propose a clustering problem that requires the use of properties emerging from the entire input set in order to be solved. The task is to cluster a set of points that lie on a number of intersecting circles, which is a generalization of the Olympic circles problem~\citep{AnandEtAl:2014}. Pairwise kernel methods for clustering perform poorly on the circles problem, whereas our ABC handles it with ease, as displayed in Figure~\ref{fig:circles}. We use the circles dataset for an ablation study in Section~\ref{subsec:experiments_circles}.

\begin{figure}%[ht]
    \includegraphics[width=\linewidth]{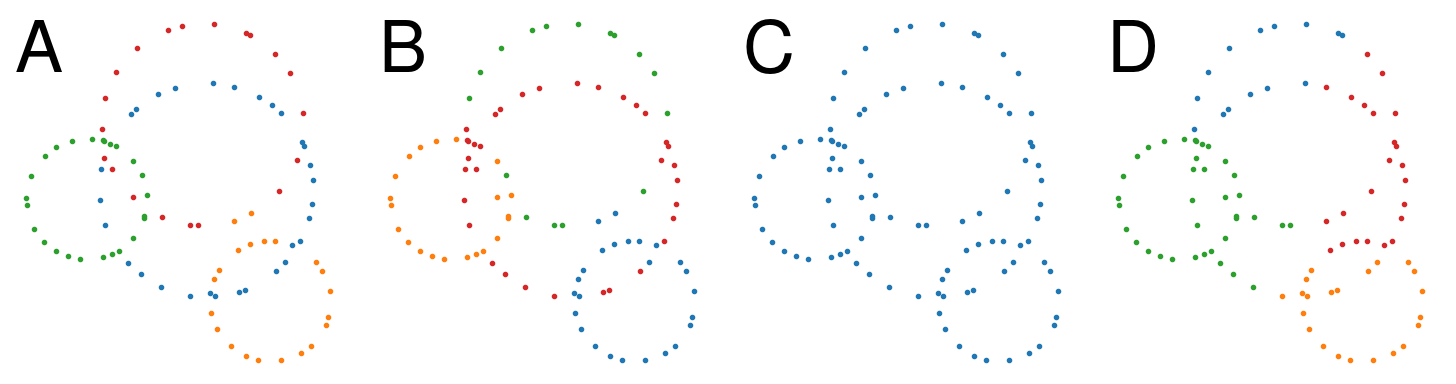}
    \caption{Illustration of the output of different clustering methods for points sampled from four overlapping circles. (A) ABC with additive attention. (B) ABC with multiplicative attention. (C) Pairwise similarity with additive attention. Pairwise similarity with multiplicative attention performed similarly. (D) Out-of-the box spectral clustering. Only D was given the true number of clusters. (Best viewed in colour.)}
    \label{fig:circles}
\end{figure}

In recent years, numerous deep neural network architectures have been proposed for clustering~\citep{xieb16:dec, MinEtAl2018deepClusteringSurvey}. The idea of using more than pairwise interactions between elements of an input set in order to improve clustering has been pursued recently in~\citet{lee2019set, lee2019deep}, and is motivated by the problem of amortized clustering~\citep{Gershman2014AmortizedInf, Stuhlmuller2013AmortisedInf}. Our architecture is inspired by the Transformer~\citep{vaswani2017attention}, which was used by~\citet{lee2019set} as the Set Transformer to improve clustering~\citep{lee2019deep}.  We inherit its benefits such as being equivariant under permutations as well as agnostic to input size. However, our approach is motivated by the use of context to improve metric learning, giving us a model that is moreover agnostic to the number of clusters. We also provide theoretical evidence that the Transformer architecture is effective for metric learning and clustering, and to our knowledge, are the first to do so.  

The idea of using deep metric learning to improve clustering has been pursued in~\citet{koch2015siamese, ZK2015siameseLike, hsu2018learning, hsu2019multiclass, HanEtAl:2019}, but without considering the use of context. We use ground truth labels, only in the form of pairwise constraints, to train a similarity kernel, making our approach an example of constrained clustering. These algorithms are often categorized by whether they use the constraints to only learn a metric or to also generate cluster labels~\citep{hsu2018learning}. Our architecture belongs to the former category, where we only use the constraints to learn a metric and rely on an unconstrained clustering process to obtain cluster labels. Despite this, we achieve nearly state-of-the-art clustering results on the Omniglot dataset, comparable to sophisticated methods that synthesize clusters, either using the constraints~\citep{hsu2018learning, hsu2019multiclass, HanEtAl:2019} or otherwise~\citep{lee2019set, lee2019deep}.

Our main contributions are:
\begin{itemize}
    \item ABC incorporates context in a general and flexible manner to improve metric learning for clustering. Our competitive results on Omniglot and our ablation study on our circles dataset provide support for the use of context in metric learning algorithms.
    \item We provide theoretical evidence of why the self-attention module in the Transformer architecture is well suited for clustering, justifying its effectiveness for this task.
\end{itemize} 

This paper is organized as follows: in Section~\ref{sec:background}, we present some recent work upon which we build our architecture, which is presented in Section~\ref{sec:architecture}. In Section~\ref{sec:analysis}, we look into some theoretical aspects of ABC, and present experimental results in Section~\ref{sec:experiments}. Then we compare against related works in Section~\ref{sec:related_works}, and we end with a discussion of our architecture in Section~\ref{sec:discussion}.

\section{Background}\label{sec:background}
Taking inspiration from kernel methods, we aim to compute a similarity matrix from a sequence of data points. Our architecture is inspired by ideas from two streams: the metric learning literature and the Siamese network~\citep{koch2015siamese} on how to learn compatibility scores, and the Transformer architecture~\citep{vaswani2017attention} and the Set Transformer~\citep{lee2019set} on how to use context to make decisions. We discuss a few concepts from the literature which will form building blocks of our architecture in the next section.

\subsection{Compatibility}

In this section we introduce some compatibility functions which compute a similarity score between two vector arguments, called the \textit{query} and \textit{key} respectively. We present the forms of compatibility used in this paper in Table~\ref{compats} and for both of these forms, keys and queries are required to have equal dimension $d$.
\begin{table}[ht]
    \caption{Possible implementations of the compatibility function. \texttt{act} is any element wise activation function, such as \texttt{tanh} or \texttt{sigmoid}.}
    \centering
    \begin{tabular}{|l|l|c|r|}
        \hline
        Form & Parameters & Expression & Reference \\
        \hline \hline
        Multiplicative & None & $\mathbf{q}^\intercal\mathbf{k}/\sqrt{d}$ & \citep{vaswani2017attention} \\\hline
        Additive & $\mathbf{w} \in \mathbb{R}^H$ & $\texttt{act}(\mathbf{q} + \mathbf{k})^\intercal \mathbf{w}$ & \citep{DBLP:journals/corr/BahdanauCB14} \\\hline
    \end{tabular}
    \label{compats}
\end{table}

In Siamese Networks~\citep{koch2015siamese}, compatibility between two input images is measured by the sigmoid of a weighted L1-distance between representations of the input. This can be seen as a special case of additive compatibility above. The Transformer~\citep{vaswani2017attention} and Set Transformer~\citep{lee2019set, lee2019deep} make use of multiplicative compatibility. 

\subsection{The Transformer}
The attention mechanism forms the core of the Transformer architecture, and generates contextually weighted convex combinations of vectors. The elements included in this combination are called values and the weights are provided via compatibilities between queries and keys as in the previous section.

Suppose we have a length $m$ sequence of query vectors and a length $n$ sequence of key-value pairs. We denote the the dimensionality of each query, key and value vector by $d_q$, $d_k$, and $d_v$ respectively. In matrix form, these are expressed as $Q \in \mathbb{R}^{m \times d_q}$ for the queries, $K \in \mathbb{R}^{n \times d_k}$ for the keys, and $V \in \mathbb{R}^{n \times d_v}$ for the values. The attention function $\texttt{Att}$ with softmax activation is given as
\begin{displaymath}
\begin{aligned}
    \texttt{Att}(Q, K, V) &= AV,\\
    \text{with } A_{i,j} &= \frac{\exp(C_{i,j})}{\sum_{k=1}^{n}\exp(C_{i,k})} \quad \text{(i.e. row wise softmax)},\\
    \text{for } C &= \texttt{compat}(Q, K) \in \mathbb{R}^{m \times n}.
\end{aligned} 
\end{displaymath}
The result is a new encoded sequence of length $m$. We use the terms additive or multiplicative attention to specify the compatibility function that a particular form of attention uses.

Multi-head Attention ($\texttt{MHA}$)~\citep{vaswani2017attention} extends the standard attention mechanism to employ multiple representations of the data in parallel.
Each query vector computes $h$ separate convex combinations over the value vectors as opposed to a single combination. The concatenation of the $h$ combinations are projected to a single vector again representing an encoding of each query with respect to the entire sequence of key-value pairs. The intuition is that each head can attend to different properties of the terms of the key-value sequence. This is functionally expressed as
\begin{displaymath}
\begin{aligned}
    \texttt{MHA}(Q, K, V) &= \texttt{concat}(O_1,\ldots,O_h)W_O,\\
    \text{with } O_j &= \texttt{Att}(QW^{(q)}_j, KW^{(k)}_j, VW^{(v)}_j) ,\,\, \text{ for } j = 1, \ldots, h.
\end{aligned}
\end{displaymath}
This formulation introduces parameters $W_O \in \mathbb{R}^{hd_v' \times d}$ and  $W^{(x)} \in \mathbb{R}^{h \times d_x \times d_x'}$, for each $x \in \{q, k, v\}$, where $d_x'$ is the desired projection length chosen as a hyper-parameter, and which is typically set to $d_x' = d_x / h$.

Another innovation in~\citet{vaswani2017attention} was the introduction of a skip connection, layer normalisation and a fully connected layer. The result is referred to as the the Multi-head Attention Block ($\texttt{MAB}$) by \citet{lee2019set}, and given by
\begin{align}
    \texttt{MAB}(Q, K, V) &= \texttt{LayerNorm}(H + \texttt{FF}(H)),\label{eq:second_skip_connection} \\
    \text{with } H &= \texttt{LayerNorm}(Q + \texttt{MHA}(Q, K, V)),\label{eq:first_skip_connection}
\end{align}
where $\texttt{FF}$ is a feed-forward layer operating element wise, and \texttt{LayerNorm} is layer normalisation~\citep{ba2016layer}.

For our purposes we will only need a special case of the MAB where the queries, keys, and values are all equal. \citet{lee2019set} denote the special case as $\texttt{SAB}(X)=\texttt{MAB}(X,X,X)$ and we will follow that notation.

\section{Architecture}\label{sec:architecture}
The ABC architecture is a composition of previously introduced components. In the most general case, ABC expects a variable-sized set of elements as input, where each element is represented by a fixed-sized feature vector. From this, ABC outputs a square matrix of the similarity scores between all pairs of elements in the input.

A note on terminology: some literature uses the word \textit{mini-batch} to mean a single input set whose elements are to be clustered. To avoid confusion with the concept of mini-batches used in training a neural network, from now on we opt to reserve the terminology \textit{input instance} instead. 

\subsection{Abstract definition}

Let $d_x$ be the dimensionality of input elements and $d_z$ be the desired number of latent features, chosen as a hyper-parameter. ABC consists of two sequential components: 
\begin{enumerate}
    \item \textbf{Embedding:} A function $\mathcal{T}$ mapping an any length sequence of elements in $\mathbb{R}^{d_x}$ to a same-length sequence of encoded elements in $\mathbb{R}^{d_z}$, or in tensor notation: for any $n\in\mathbb{N}$ we have $\mathcal{T}:\mathbb{R}^{n\times d_x}\to\mathbb{R}^{n\times d_z}$.
    \item \textbf{Similarity:} A kernel function $\kappa: \mathbb{R}^{d_z} \times \mathbb{R}^{d_z} \rightarrow \mathbb{R}$,
\end{enumerate}
such that for $X \in \mathbb{R}^{n \times d_x}$ the output is an $n\times n$-matrix. Explicitly, composing these parts gives us for any $n\in\mathbb{N}$ a function $\textrm{ABC} : \mathbb{R}^{n \times d_x} \rightarrow \mathbb{R}^{n \times n}$ with
\begin{displaymath}
\textrm{ABC}(X)_{i,j} = \kappa(\mathcal{T}(X)_i, \mathcal{T}(X)_j).
\end{displaymath}

\subsection{Explicit embedding and similarity}

We construct the embedding layer by composing a fixed number of SABs:
\begin{displaymath}
    \mathcal{T}(X) = (\texttt{SAB}_1 \circ \cdots \circ \texttt{SAB}_N)(X)
\end{displaymath}

and we rely on the embedding stage to capture the relevant information related to all terms of the input instance and encode that within every term of its output. As such, computing the similarity can simply be performed pairwise. We now make the choice to constrain the output of the similarity function $\kappa$ to lie in the unit interval. Our choice for the symmetric similarity component is
\begin{align*}
    \kappa(\mathbf{z}_i, \mathbf{z}_j) = \frac{1}{2} \left[ \texttt{sigmoid}(\texttt{compat}(\mathbf{z}_i, \mathbf{z}_j)) + \texttt{sigmoid}(\texttt{compat}(\mathbf{z}_j, \mathbf{z}_i)) \right],
\end{align*}
where $\mathbf{z}_i$ is the $i$th term of the encoded sequence.

\subsection{Loss function and training}

Given a labelled input instance comprised of a collection of elements and corresponding cluster labels, we train ABC in a supervised manner using a binary ground-truth matrix indicating same-cluster membership. Each cell of the output matrix can be interpreted as the probability that two elements are members of the same cluster. The loss is given as the mean binary cross entropy (BCE) of each cell of the output matrix.

\subsection{Supervised kernel to unsupervised clustering} 

ABC learns a mapping directly from an input instance to a kernel matrix. We pass this matrix in to an off-the-shelf kernel-based clustering method, such as spectral clustering, to obtain the cluster labels.

What remains is to specify the number of clusters present in the predicted kernel. Depending on the use-case this can be supplied by the user or inferred from the kernel matrix by using the eigengap method~\citep{von2007tutorial}. Let $A$ be the symmetric kernel matrix. The number of clusters inferred from this matrix is 
\begin{displaymath}
    \textrm{NumClusters}(A) = \textrm{argmax}_{i\in\{1,\ldots,n\}} \{ \lambda_i - \lambda_{i+1} \},
\end{displaymath}
where $\lambda_i$ is the $i$th largest eigenvalue of the normalized Laplacian $$L = I -  D^{-\frac{1}{2}}AD^{-\frac{1}{2}},$$ and where $D$ is the diagonal degree matrix of $A$.

\section{Analysis}\label{sec:analysis}

In this section we discuss some theoretical properties of the architecture. We focus on the role of attention and the effects of skip-connections~\citep{resnet}. In particular, we show how these elements are able to separate clusters from other clusters, making it easier for the similarity block of ABC to learn pairwise similarity scores based on the context given by the entire input instance.

We consider a simplified version of the SAB using just a single attention head. It is not difficult to prove that attention with any compatibility function maps a set of vectors into its convex hull, and that the diameter of the image is strictly smaller than the diameter of the original (see Appendix~\ref{subsec:appendix_attention_dyn_sys} for details). This leads repeated application to blur the input data too much to extract relevant features. This behaviour is also noticed in~\citet{bello2016neural} and is counteracted in the Transformer by the use of skip-connections. Reports showing that skip-connections play a role in preserving the scale of the output in feed-forward networks can for example be found in~\citet{pmlr-v70-balduzzi17b, zaeemzadeh2018normpreservation}, and we include a short discussion on the same effect in our setting in Appendix~\ref{subsec:appendix_skip_connections}.
We note that the remaining parts of the Multi-Head attention block as described in equations~\eqref{eq:second_skip_connection} and~\eqref{eq:first_skip_connection}, i.e. the layer normalizations and the element wise feed-forward layer, are of a `global' nature, by which we mean that they do not depend on different elements in the input instance. These parts merely support the functionality of the network along more general deep learning terms and they do not form an interesting component to this particular analysis.

The counterbalanced contraction discussed above holds for the entire dataset as a whole, but more structure can be uncovered that motivates the use of the set encoder in our architecture. Somewhat informally we may state it as the following, of which the formal statement and proof are treated in Appendix~\ref{subsec:appendix_formal_example}.

\begin{proposition}\label{prop:two_clusters_informally}
Assume we are given a set of points that falls apart into two subsets $A$ and $B$, where the pairwise compatibility weights within each of $A$ and $B$ are larger than the pairwise weights between $A$ and $B$. Under repeated application of SABs and under some symmetry conditions, the two subsets become increasingly separated.
\end{proposition}
\citet{AnandEtAl:2014} use a similar idea to devise a transformation for their kernel. A linear transformation is designed to bring pairs of points from a cluster closer together and to push pairs of points from different clusters apart, by iterating over all labelled pairs. The Transformer architecture accomplishes this without the restriction of linearity and without the need for iteration over points in an input instance due to an amortization of the clustering process.

\section{Experiments}\label{sec:experiments}
We conduct two experiments to validate the feasibility of our architecture and to evaluate the claim that context helps learn good similarity output. We give details on how we sample training instances in Appendix~\ref{sec:appendix_sampling}.

\subsection{Toy Problem: Points on a circle}\label{subsec:experiments_circles}

To generalize the phenomenon of real-world datasets intersecting, such as characters in multiple languages, as well as to illustrate the necessity for context during some clustering tasks, we devise the following toy problem. Given a fixed-length sequence of points, where each point lies on four likely overlapping circles, cluster points according to the circle they lie on. As we will demonstrate, only considering the pairwise similarities between points is insufficient to solve this problem, but our architecture does give a satisfactory solution. 

We try two variants of ABC, one with additive attention and the other with multiplicative attention. As an ablation study, we compare against a generic pairwise metric learning method as well as out-of-the-box spectral clustering. For the pairwise metric learning method, we remove the embedding block and use only the similarity block. By comparing with spectral clustering, we show the improvement that our architecture brings.

In Figure~\ref{fig:circles2}, we present the adjusted Rand score of all these clustering methods for different values of input instance length. Notice that the pairwise method performs poorly, in fact worse than out-of-the-box spectral clustering. The multiplicative and additive variants of ABC far outperform the other two methods on the circles problem, thus validating our use of context in learning a metric.

\begin{figure}[ht]
    \centering
    \includegraphics[width=0.75\linewidth]{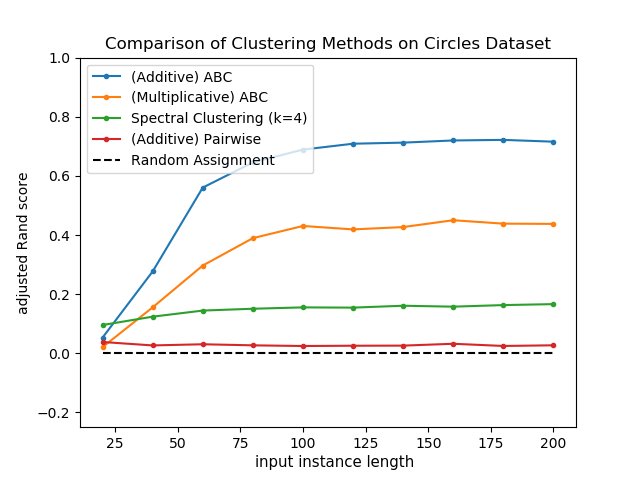}
    \captionof{figure}{Comparative performance on the circles problem of ABC with either additive or multiplicative attention, as well as ablated versions of the ABC architecture. The horizontal axis shows the number of points sampled from the combined circles. The vertical axis shows the Rand score adjusted so that random assignment gives a score of 0. The big gap in performance between pairwise and spectral clustering on the one hand and the two versions of ABC on the other shows the benefit that context brings.}
    \label{fig:circles2}
\end{figure}

\subsection{Omniglot Clustering}
The Omniglot training dataset~\citep{omniglot_dataset} consists of images of characters from the alphabets of 30 languages, with another 20 alphabets reserved for testing. Each alphabet has varying numbers of characters, each with 20 unique example images. This dataset was proposed to test model performance on one-shot learning tasks~\citep{LAKE2019omniglotReport}, where a model must learn from single examples of novel categories. We attempt clustering of images from novel classes within alphabets. We treat each character as a class such that an alphabet is a grouping of related classes.

Before attempting the above task, it is critical for a model to first learn how to learn from a limited collection of examples from novel categories. Doing so will teach the model how to extract general information that can be applied to categories never before seen during training. 

\begin{figure}%[htbp]
  \centering
  \includegraphics[width=0.75\linewidth]{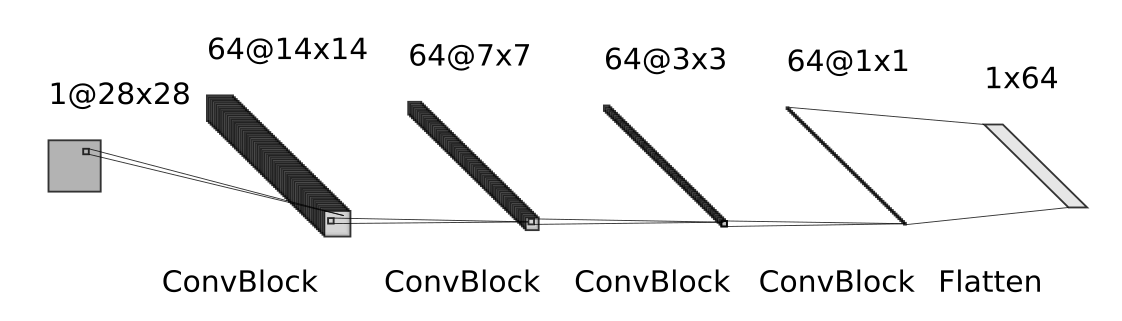}
  \caption{Illustration of the architecture used to encode each character image prior to computing the similarity matrix. Each ConvBlock performs in sequence: Conv2d with padding $1$ and kernel size $3$, batch normalization, ReLU, and max pooling with kernel size $2$.}
  \label{fig:enc}
\end{figure}

For training, each input instance consists of 100 within alphabet images, where the number of unique characters per input instance varies as much as permitted by the available data.
We use the CNN from~\citet{Vinyals2016_matchingNetworks} as the image embedding function. This module is illustrated in Figure~\ref{fig:enc}. 
Training is conducted using our implementation in PyTorch\footnote{Code will be available at \url{https://github.com/DramaCow/ABC}.} and uses the standard Adam optimizer. Details of the hyperparameters can be found in Appendix~\ref{subsec:appendix_experiment_hyperparams}.

\begin{figure}%[ht]
    \centering
    \includegraphics[width=0.75\linewidth]{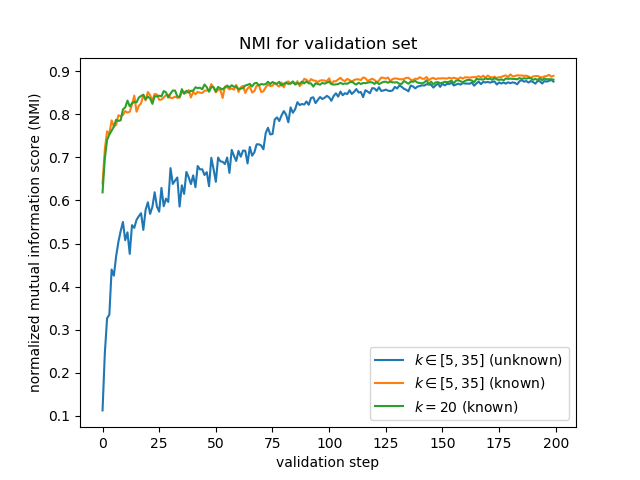}
    \captionof{figure}{Clustering performance on the test set for our three clustering tasks on Omniglot over the course of training. As the model improves, estimating the number of clusters becomes more accurate, and the disadvantage of not knowing the true number of clusters becomes negligible. }
    \label{fig:nmi}
\end{figure}

For testing, we use the 20 alphabets from the reserved lot in Omniglot, as a standalone dataset each. At test time, an instance of 100 images are presented to the model, assembled as a random number of elements chosen from a certain number of clusters as described below. We report clustering performance on three tasks with: 
\begin{enumerate*}[label=(\roman*)]
\item a variable number of clusters, unknown at inference, \item a variable number of clusters, known at inference, and \item a fixed number of clusters ($k=20$), known at inference. 
\end{enumerate*}
Note that training is independent of the task; at inference time, all tasks use the same trained model.

\begin{table}
    \centering
    \caption{ABC results on the three tasks outlined in this section. The scores displayed are the means over the 20 testing alphabets. The per-alphabet split can be found in Appendix~\ref{sec:appendix_omniglot_details}.}
    \label{tab:ABC_results}
    \begin{tabular}{|l|c|}
        \hline
        Task & NMI \\
        \hline
        \hline
         Variable unknown number of clusters & 0.874\\
         Variable known number of clusters & 0.893\\
         Fixed number of clusters ($k=20$) & 0.884\\
         \hline
    \end{tabular}
\end{table}

Our results show that ABC performs equally well on all three tasks, see Table~\ref{tab:ABC_results}. In particular, the Normalized Mutual Information score (NMI) obtained with an unknown number of clusters matches the values that are obtained when the number of clusters is known. Hence, after training the model to convergence, it is not necessary to know the true number of clusters to obtain good performance. 

In Table~\ref{tab:comparison}, we compare against previous results reported on this problem. In this table, there are two categories of clustering methods; the first four methods use supervised metric learning in combination with unsupervised clusterers, whereas the last four methods use the constraints to synthesize clusters, which adds to the model complexity. ABC belongs to the former category, but performs comparably to the latter category of clustering methods. Also notice that ABC with multiplicative compatibility outperforms the only other method that uses context, distinguished by the $\dagger$ symbol added to its name in Table~\ref{tab:comparison}. This validates our hypothesis that context can improve metric learning, and that using context can be valuable when working with real world data. 

\begin{table}
\centering
\caption{Comparative results on Omniglot. The table presents results for known and unknown number of clusters. Where the architecture relies on knowning a (maximum) number of clusters, such as KLC, that maximum is set to 100. The first four entries are copied from~\citet{hsu2018learning} as their methods are most relevant in comparison to ours. The table is split up as explained in the main text.}
\label{tab:comparison}
\begin{tabular}{|l|c|c|r|}
    \hline
    Method & NMI (known) & NMI (unk.) & Reference\\
    \hline
    \hline
    ITML & 0.674 &  0.727 &\citep{10.1145/1273496.1273523}\\
    SKMS & - & 0.693 &\citep{AnandEtAl:2014} \\
    SKKm & 0.770 & 0.781 &\citep{AnandEtAl:2014} \\
    SKLR & 0.791 & 0.760 &\citep{AmidEtAl:2016}\\
    \hline
    ABC (add. compat.)$^\dagger$ & 0.873 & 0.816 & (ours)\\
    ABC (mul. compat.)$^\dagger$ & 0.893 & 0.874 & (ours)\\
    \hline
    DAC$^\dagger$ & - & 0.829 & \citep{lee2019deep}\\
    KLC & 0.889 & 0.874 &\citep{hsu2018learning}\\
    MLC & 0.897 & 0.893 &\citep{hsu2019multiclass}\\
    DTC-$\Pi$ & 0.949 & 0.945 &\citep{HanEtAl:2019} \\
    \hline
\end{tabular}
\end{table}

\section{Related works}\label{sec:related_works}

Our method is similar to a line of research where a distance metric, rather than a similarity score, is learned in a supervised manner, which can then be used as input to off-the-shelf clustering methods, such as $K$-means~\citep{xing_ng_jordan_russell, 10.1145/1015330.1015376, 10.1145/1273496.1273523}. This line of work differs from ours in the sense that only a certain class of distances are learned\footnote{namely Mahalanobis distances} whereas our similarity scores are only restricted by the class of functions that our architecture is able to model. This is still an open research question because the class of functions that the Transformer can model has only partly been studied~\citep{DBLP:conf/iclr/YunBRRK20}.  

Deep neural nets have been 
used to learn a pairwise metric in numerous works~\citep{ZK2015siameseLike, hsu2018learning, wojke2018deep, hsu2019multiclass}, most notably in the Siamese network~\citep{koch2015siamese}. The idea of using contextual information has not been explored in any of these papers. 

Many models go further than metric learning by also learning how to synthesize clusters. An example of constrained clustering can be found in~\citet{AnandEtAl:2014}, where pairwise constraints are used to linearly transform a predefined kernel in an iterative manner, which is used in a kernel mean shift clustering algorithm. The kernel matrix needs to be updated iteratively for each constraint, making the algorithm difficult or even impossible to converge. An extension of this work to handle relative distances between pairs of data points can be found in~\citet{AmidEtAl:2016}.

Constrained clustering algorithms have been implemented using deep neural nets as well. In~\citet{hsu2018learning, hsu2019multiclass}, the authors train a similarity metric and transfer learning to a secondary clustering model. Both models are trained using only pairwise constraints, and any available context information remains unused in both components of their architecture. In~\citet{HanEtAl:2019}, a constrained clusterer inspired by the deep embedded clustering idea~\citep{xieb16:dec} is proposed, along with a number of best practices such as temporal ensembling and consistency constraints in the loss function. These techniques are fairly generic and can perhaps be applied to any other clustering algorithm to improve its results. Their model generates clusters by slowly annealing them, requiring optimization and back-propagation even at test time. The models from~\citet{hsu2018learning} and~\citet{hsu2019multiclass} also have this requirement. This may not be feasible during deployment.

The Set Transformer architecture~\citep{lee2019set} uses the Transformer as a contextual encoder, followed by a pooling layer that uses a fixed number of seed vectors as queries. This architecture is used to cluster a mixture of Gaussians, but is less flexible than ours for two reasons: it requires the number of clusters in advance in setting the number of seed vectors, and those seed vectors being learned makes their approach less adaptable to unseen classes. The first limitation is addressed in a follow-up paper~\citep{lee2019deep}. Our architecture, due to its use of metric learning in place of the pooling layer with learned seed vectors, is inductive and can handle new classes with ease. We also present a mathematical justification for the use of the Transformer in clustering applications. 
\section{Discussion}\label{sec:discussion}
It is perhaps unsurprising that the Transformer architecture performs well for clustering in addition to a number of other areas. The self-attention module in the Transformer architecture offers a unique advantage to neural networks: this module acts as a linear layer whose weights are determined by the compatibility scores of the queries and keys rather than a fixed set of learned values. This makes the self-attention module a nonparametric approximator~\citep{nonparametricStatisticsBook, Orbanz2010BayesianNonparametrics}, whose expressivity is far more than what might be expected by looking at the parameter reuse in the compatibility module~\citep{DBLP:conf/iclr/YunBRRK20}. 

The encoder in ABC can be seen to be balancing the two objectives of using context and learning from ground truth labels, in the manner in which it combines the multi-head attention term with a skip-connection. This sometimes gives rise to conflicts, as seen in the example in Figure~\ref{fig:omniglot_eg}. Here, the input instance consists of all the variations of the letter k. The predicted similarity matrix is far from the ground truth: a perceived mistake by the model. Upon closer look however, we can see that while each element represents the same character, each of them is written in a slightly different way. For this particular input instance, those small differences are precisely what makes up the relevant context, and the model is able to pick up on that. To accommodate for situations where the level of context should be balanced against the relevance of ground truth labels, one could imagine a modified version of the Transformer using weighted skip-connections as in Highway Networks~\citep{NIPS2015_5850}. The attention weighted average brings context into the prediction and the skip-connections carry through the information coming from individual data points. The extra weights would allow the model to learn when to focus on context and when to ignore it.

\begin{figure}[ht]
\begin{center}
    \includegraphics[width=0.75\linewidth]{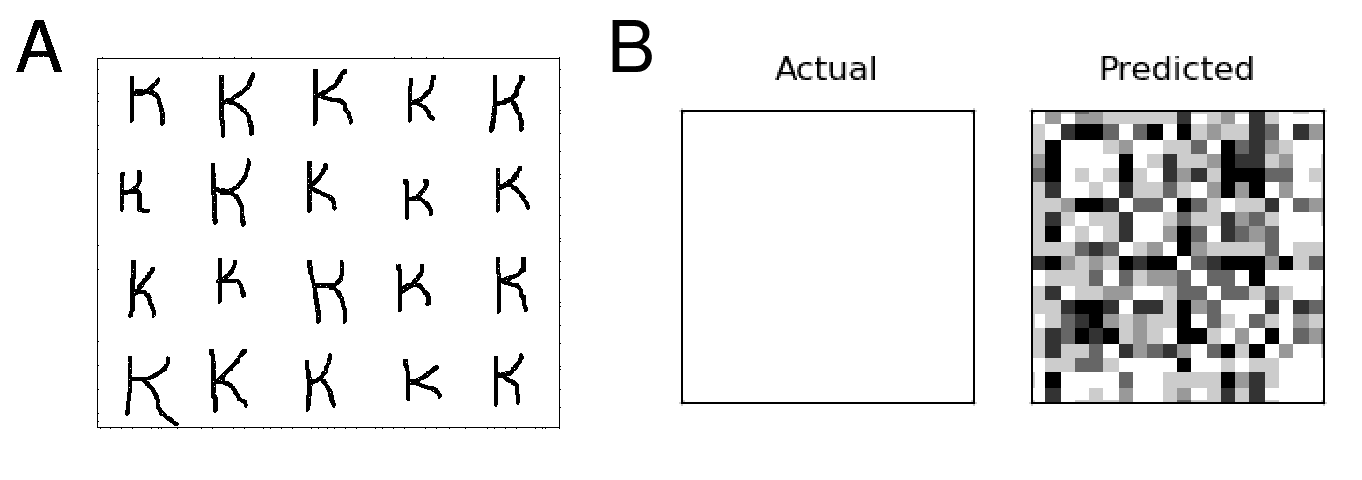}
    \caption{(A) Example input instance of characters all of the same class. (B) Ground-truth and predicted similarity matrices and their difference in greyscale, where white means a value of 1 and black a value of 0. ABC picks up on the small differences between each of the characters; this is precisely the context that this input instance provides.}
    \label{fig:omniglot_eg}
\end{center}
\end{figure}

In this paper we have only addressed a single interpretation of `context', one that is implicit. We could however be presented with a dataset in which \textit{explicit} context is available which may take the form of a weighted adjacency matrix rather than merely binary ground-truth cluster labels. This form of context is also naturally handled by a reimagining of the ABC architecture as a graph neural network~\citep{Scarselli2009GNNs, GNNsurvey2019}. We would replace the encoder stage with a graph attention network~\citep{velickovic2018graph} that incorporates weighted adjacency data. This generalizes the architecture of ABC, where the present architecture can be interpreted to act on a fully connected graph with unit weights.

So far, the use of constraints has been limited to learning a similarity kernel in ABC, in contrast to the approach taken in~\citet{hsu2018learning}. A hybrid approach where the similarities are learned instance wise, like in ABC, and then processed using a learned model which is robust to noise would be an interesting avenue for future research. We would also be interested to see how far we can push our method by including general good practices as in~\citet{HanEtAl:2019}.

\bibliography{refs}
\bibliographystyle{plainnat}

\newpage
\appendix
\section{More details on the Analysis}\label{sec:appendix_analysis}

\subsection{Attention as a dynamical system}\label{subsec:appendix_attention_dyn_sys}
This section deals with the equation
\begin{align}
x_{i, t+1} &= \sum_{j=1}^{n} w_{i,j,t} x_{j,t}, \label{eq:dyn_rep}\\
\textrm{with } w_{i,j,t} &= \textrm{softmax}((\textrm{compat}(x_{i,t}W^{(q)}, x_{\ell,t}W^{(k)}))_{\ell=1}^n)_j \nonumber
\end{align}
which is an abstraction of single headed self attention, i.e. \texttt{MHA} with one attention head and where the sets of keys, queries, and values are all equal. Note that all weights are positive and that for any fixed $i$ and $t$ the weights $w_{i,j,t}$ sum to~1. For any $t\in\mathbb{N}$ we write $X_t=\{x_{1,t},\ldots,x_{n,t}\}$. Equation~\eqref{eq:dyn_rep} may be interpreted as a discrete time dynamical system, where successive time steps correspond to the forward passes through composed attention blocks. 

\begin{lemma}\label{lem:convex_hull}
For the dynamical system described in (\ref{eq:dyn_rep}) the convex hull of $X_{t+1}$ is contained in the convex hull of $X_t$.
\end{lemma}
\begin{proof}
Equation \eqref{eq:dyn_rep} gives each term in $X_{t+1}$ as a convex combination of terms in $X_t$. The result follows since a convex hull is closed under taking convex combinations of any of its elements. 
\end{proof}

\begin{lemma}\label{lem:diam}
For any positive lower bound $\delta_t$ on the weights $w_{i,j,t}$ at time step $t$, the diameter of the set of points decreases as
\begin{displaymath}
\textrm{diam}(X_{t_1}) \le (1 - 2\delta_t)\textrm{diam}(X_t).
\end{displaymath}
\end{lemma}
\begin{proof}
Let $\pi$ be orthogonal projection onto any line in $\mathbb{R}^{d_x}$. Under repara\-metrization of the line, we may assume $\pi(X_t)\subseteq [0,d_{\pi,t}]$ to hold such that neither $\pi^{-1}(0)$ nor $\pi^{-1}(d_{\pi,t})$ are empty. 
Consider the extremal situation with $\# \pi^{-1}(0) = n-1$ and $\#\pi^{-1}(d_{\pi,t})=1$. Without loss of generality we write $\pi(x_{1,t})=d_{\pi,t}$. For any positive lower bound $\delta_t$ on the weights $w_{i,j,t}$ and by linearity of $\pi$, we conclude that we have
\begin{displaymath}
\pi(x_{i,t+1}) = \sum_{j=1}^n w_{i,j,t}\pi(x_{j,t}) = w_{i,1,t}d_{\pi,t} \geq \delta_t d_{\pi,t}.
\end{displaymath}
For the reverse extremal situation with $\# \pi^{-1}(0) = 1$ and $\#\pi^{-1}(d_{\pi,t})=n-1$, we write without loss of generality $\pi(x_{1,t})=0$. Analogous to before we conclude
\begin{displaymath}
\pi(x_{i,t+1}) \leq (1-\delta_t)d_{\pi,t}
\end{displaymath}
by the fact that for any $i$, the sum $\sum_{j=2}^n w_{i,j,t}$ is bounded above by $1-\delta_t$.

Any other alignment of the projection images is less extreme, giving rise to 
\begin{equation}\label{eq:interval_inclusion}
\pi(X_{t+1}) \subseteq [\delta_t d_{\pi,t}, (1-\delta_t)d_{\pi,t}].
\end{equation}

The above holds for any projection, so in particular we have $d_{\pi,t}\leq \textrm{diam}(X_t)$. Now consider two points in $X_{t+1}$ with maximal distance $\textrm{diam}(X_{t+1})$, and in particular consider the projection $\pi$ onto the line defined by these two points. Then we have
\begin{displaymath}
\textrm{diam}(X_{t+1}) = \textrm{diam}(\pi(X_{t+1})) \leq (1-2\delta_t)\textrm{diam}(X_t),
\end{displaymath}
having used equation~\eqref{eq:interval_inclusion} for the inequality.
\end{proof}

Note that the above proof by considering extremal situations may initially seem to be at odds with the stronger result of
\begin{displaymath}
\textrm{diam}(X_{t+1}) \leq \left(1 - \frac{n\delta_t}{4}\right) \textrm{diam}(X_t)
\end{displaymath}
that may be derived following arguments in~\citet{linderman2019clustering}. This apparent paradox is resolved by realizing that the two extremal situations we describe can never occur simultaneously unless $n=2$ holds. In that particular situation, our bound is better.

As mentioned in Section~\ref{sec:analysis}, Lemmas~\ref{lem:convex_hull} and~\ref{lem:diam} together imply that if the Transformer architecture would not include skip-connections, then the diameter of the set of input vectors would shrink with each successive attention block. How skip-connections counteract this specifically for our clustering purposes is further discussed in Appendix~\ref{subsec:appendix_skip_connections}.

\subsection{Formal treatment of Proposition \ref{prop:two_clusters_informally}}\label{subsec:appendix_formal_example}

In this section we will analyse an extension of equation~\eqref{eq:dyn_rep} to also include skip-connections, after which we will specify to clustering in Corollary~\ref{cor:fast_and_slow}.

Let $n$ and $m$ be two positive integers. We will write $I_A=\{1,\ldots,n\}$ and $I_B=\{n+1,\ldots,n+m\}$. Consider the discrete time dynamical system on a set of points $x_{i,t}\in \mathbb{R}^d$ for $i\in I_A\cup I_B$, $t\in\mathbb{N}$ and some $d\geq 0$, given by the update rule

\begin{equation}\label{eq:formal_eg_dyn_sys}
\Delta x_{i,t+1} := x_{i,t+1} - x_{i,t} = \sum_{j\in I_A\cup I_B}w_{i,j,t}x_{j,t}
\end{equation}

under the following assumptions:
\begin{align*}
w_{i,j,t} = \alpha_t>0 &\textrm{ for }i,j\in I_A,\ i\neq j,\\
w_{i,j,t} = \beta_t>0 &\textrm{ for }i,j\in I_B,\ i\neq j,\\
w_{i,j,t} = \gamma_t>0 &\textrm{ for }i\in I_A,\ j\in I_B,\\
w_{i,i,t} = \delta_t>0 &\textrm{ for }i\in I_A\cup I_B.
\end{align*}

Assume for any $i\in I_A\cup I_B$ and $t\in\mathbb{N}$ moreover
\begin{equation}\label{eq:sum_one}
\sum_{j\in I_A\cup I_B} w_{i,j,t} = 1.
\end{equation}

Notice that this is the setup as described informally in Proposition \ref{prop:two_clusters_informally}, for the two clusters given by $A=\{x_{i,0}:i\in I_A\}$ and $B=\{x_{i,0}:i\in I_B\}$. The use of skip-connections is visible in equation~\eqref{eq:formal_eg_dyn_sys} yielding $\Delta x_{i,t+1}$ rather than $x_{i,t+1}$ itself.

We will write
\begin{displaymath}
c_{p,t} = \frac{1}{\# I_p} \sum_{i\in I_p} x_{i,t} \textrm{ for } p=A,B
\end{displaymath}
for the centroids of the two clusters.

We will assume $\delta_t> \max\{\alpha_t,\beta_t\}$ for all $t\in\mathbb{N}$. This assumption is natural in our application domain of similarity scores, and it will in fact be necessary in Corollary \ref{cor:fast_and_slow}. While not strictly necessary for the proof of Proposition \ref{prop:two_clusters} itself, we already assume it now so that the quantities involved in the statement of the proposition are non-negative.

\begin{proposition}\label{prop:two_clusters}
Using the notation and assumptions outlined above, the following statements hold:
\begin{enumerate}
	\item\label{enum:two_clusters1} For all $i,j\in I_A$ and $t\in\mathbb{N}$ we have $x_{i,t+1} - x_{j,t+1} = (1+\delta_t - \alpha_t)(x_{i,t}-x_{j,t})$.
	\item\label{enum:two_clusters2} For all $i,j\in I_B$ and $t\in\mathbb{N}$ we have $x_{i,t+1} - x_{j,t+1} = (1+\delta_t - \beta_t)(x_{i,t}-x_{j,t})$.
	\item\label{enum:two_clusters3} For all $t\in\mathbb{N}$ we have $c_{1,t+1} - c_{2,t+1} = (2-(n+m)\gamma_t)(c_{1,t} - c_{2,t})$.
\end{enumerate}
\end{proposition}

Note before we start the proof itself, that expanding \eqref{eq:sum_one} for $i\in I_A$ and $i\in I_B$ separately gives relations between the different weights:
\begin{equation}\label{eq:observation}
\begin{split}
\delta_t + (n-1)\alpha_t + m\gamma_t &=1, \textrm{ and}\\
\delta_t + (m-1)\beta_t + n\gamma_t &=1.
\end{split}
\end{equation}

\begin{proof}[Proof of Proposition \ref{prop:two_clusters}]
The proofs of parts \ref{enum:two_clusters1} and \ref{enum:two_clusters2} are identical up to switching the roles of $I_A$ and $I_B$, so we merely give the former, which is by simple computation. For $i,j\in I_A$ we have
\begin{displaymath}
\Delta x_{i,t+1} - \Delta x_{j,t+1} = \sum_{\ell\in I_A} w_{i,\ell,t}x_{\ell,t} + \sum_{\ell\in I_B} w_{i,\ell,t}x_{\ell,t} - \sum_{\ell\in I_A} w_{j,\ell,t}x_{\ell,t} - \sum_{\ell\in I_B} w_{j,\ell,t}x_{\ell,t}.
\end{displaymath}
Notice that the second and fourth sum both equal $\gamma_t\sum_{\ell\in I_B}x_{\ell,t}$. As they have opposite signs, these two sums disappear from the overall expression. Similarly, each term in the first and third sum that corresponds to some $\ell\in I_A\setminus\{i,j\}$ occurs with opposite signs in the overall expression and hence disappears. Therefore we arrive at
\begin{displaymath}
\Delta x_{i,t+1} - \Delta x_{j,t+1} = w_{i,i,t}x_{i,t} + w_{i,j,t}x_{j,t} - w_{j,i,t}x_{i,t} - w_{j,j,t}x_{j,t},
\end{displaymath}
which equals $(\delta_t-\alpha_t)x_{i,t} + (\alpha_t-\delta_t)x_{j,t} = (\delta_t-\alpha_t)(x_{i,t}-x_{j,t})$. Retrieval of the statement of the proposition follows by expanding $\Delta x_{i,t+1} = x_{i,t+1}-x_{i,t}$, giving rise to the additional $1$ inside the parentheses.

For the proof of part \ref{enum:two_clusters3} we notice that we may write 
\begin{equation}\label{eq:centroid_difference}
c_{1,t+1} - c_{2,t+1} = \frac{1}{nm}\sum_{i\in I_A,j\in I_B}x_{i,t+1}-x_{j,t+1}
\end{equation}
for all $t\in\mathbb{N}$, so we first study the individual differences $x_{i,t+1}-x_{j,t+1}$ for $i\in I_A$ and $j\in I_B$. 

Again, straightforward computation yields
\begin{align*}
\Delta x_{i,t+1} - \Delta x_{j,t+1} =& \sum_{\ell\in I_A} \left(w_{i,\ell,t} - w_{j,\ell,t}\right)x_{\ell,t} + \sum_{k\in I_B}\left(w_{i,k,t}-w_{j,k,t}\right)x_{k,t}\\
=& (\delta_t-\gamma_t) x_{i,t} + \sum_{i\neq\ell\in I_A}(\alpha_t-\gamma_t)x_{\ell,t}\\ 
&\ + (\gamma_t- \delta_t) x_{j,t} + \sum_{j\neq k\in I_B}(\gamma_t - \beta_t)x_{k,t}\\
=& (\delta_t-\gamma_t)(x_{i,t} - x_{j,t})\\
&\ + \sum_{i\neq\ell\in I_A}(\alpha_t-\gamma_t)x_{\ell,t} - \sum_{j\neq k\in I_B}(\beta_t - \gamma_t)x_{k,t}
\end{align*}
and substitution into \eqref{eq:centroid_difference} together with expansion of $\Delta x_{i,t+1}$ allows us to write
\begin{align*}
c_{1,t+1} - c_{2,t+1} =& (1+\delta_t-\gamma_t)(c_{1,t}-c_{2,t})\\ 
&\ + \frac{1}{mn}\sum_{i\in I_A,j\in I_B}\left(\sum_{i\neq\ell\in I_A}(\alpha_t-\gamma_t)x_{\ell,t} - \sum_{j\neq k\in I_B}(\beta_t - \gamma_t)x_{k,t}\right).
\end{align*}
Let us investigate the double sum here. Each term involving $x_{\ell,t}$ for $\ell\in I_A$ occurs $m(n-1)$ times since for any fixed $j\in I_B$, among the $n$ outer terms involving $i\in I_A$, it happens exactly once that there is no term involving $x_{\ell,t}$. Similarly for the terms involving $x_{k,t}$ for $k\in I_B$, which each occur $n(m-1)$ times. Hence the double sum equals
\begin{displaymath}
m(n-1)(\alpha_t-\gamma_t)\sum_{i\in I_A}x_{i,t} - n(m-1)(\beta_t-\gamma_t)\sum_{j\in I_B}x_{j,t}.
\end{displaymath}
Accounting for the factor $\tfrac{1}{nm}$ and reinserting the definition of $c_{1,t}$ and $c_{2,t}$ we arrive at
\begin{displaymath}
c_{1,t+1} - c_{2,t+1} = \left(1+\delta_t + (n-1)\alpha_t - n\gamma_t\right)c_{1,t} - \left(1+\delta_t + (m-1)\beta_t - n\gamma_t\right)c_{2,t}.
\end{displaymath}
To finalize the proof we make use of our earlier observation from \eqref{eq:observation} that allows us to recognize that the coefficients for $c_{1,t}$ and $c_{2,t}$ in the last line are in fact equal (up to sign) and have the values $\pm (2-(n+m)\gamma_t)$. 
\end{proof}

The proposition above does not yet include one of the assumptions that were outlined in the informal statement, namely that the weights within either cluster are larger than the weights between clusters, i.e. $\gamma_t < \min\{\alpha_t, \beta_t\}$. Adding this assumption to the formalism leads us to the following corollary.

\begin{cor}\label{cor:fast_and_slow}
For any $t\in\mathbb{N}$, if $\alpha_t>\gamma_t$ holds, then at time $t$ the diameter of $\{x_{i,t}:i\in I_A\}$ expands at a slower rate than the rate at which the centroids $c_{A,t}$ and $c_{B,t}$ are pushed apart. Moreover, the same statement holds when replacing $\alpha_t$ by $\beta_t$ and $I_A$ by $I_B$.
\end{cor}
\begin{proof}
We will only give the proof for the former statement. The proof of the latter statement is identical after performing the symbolic replacement as indicated.

The rates mentioned in the corollary are $1+\delta_t-\alpha_t$ and $2-(n+m)\gamma_t$ respectively. Their ratio equals
\begin{displaymath}
\frac{1+\delta_t-\alpha_t}{2-(n+m)\gamma_t} = \frac{2 - n\alpha_t - m\gamma_t}{2-n\gamma_t-m\gamma_t},
\end{displaymath}
which is evidently smaller than 1 in case $\alpha_t>\gamma_t$ holds. Moreover, both rates are strictly lower bounded by 1, so the respective diameters grow and so does the separation between the cluster centroids.
\end{proof}

\subsection{The use of skip-connections}\label{subsec:appendix_skip_connections}
As noted in Section~\ref{sec:analysis}, the skip-connections serve a specific purpose in the Set Transformer architecture, which we discuss in a little more detail here. We will focus specifically on their use in the proofs of Proposition~\ref{prop:two_clusters} and Corollary~\ref{cor:fast_and_slow}. 

Without skip-connections, equation~\eqref{eq:formal_eg_dyn_sys} becomes
\begin{displaymath}
x_{i,t+1} = \sum_{j\in I_A\cup I_B} w_{i,j,t}x_{j,t}
\end{displaymath}
and the statement of Proposition~\ref{prop:two_clusters} would be modified. The multiplication factors $1+\delta_t-\alpha_t$ and $1+\delta_t-\beta_t$ from the first and second statements and $2-(n+m)\gamma_t$ from the third statement would each decrease by 1. This would mean that these factors would fall into the interval $(0,1)$ and each encoder block would operate in a contractive way. While the result of Corollary~\ref{cor:fast_and_slow} would remain morally correct -- each cluster would contract faster than the rate at which the cluster centroids would come together -- this would complicate training a network containing multiple stacked encoder blocks. 

\newpage
\section{More details on the sampling procedure}\label{sec:appendix_sampling}

Given a classification dataset containing a collection of examples with corresponding class labels, we briefly outline a general procedure to synthesize an ABC-ready dataset. A single input instance is independently generated using the procedure outlined in Algorithm~\ref{alg:gen}.

\begin{algorithm}[ht]
\caption{Generating a cluster instance from a classification dataset}
\label{alg:gen}

\SetKwInOut{Input}{input}
\SetKwInOut{Constraint}{constraint}
\SetKwInOut{Output}{output}
\SetKwFunction{Uniform}{uniform}
\SetKwFunction{Min}{min}
\SetKwFunction{Return}{return}

\Input{desired length of output sequence $L$}
\Constraint{number of classes $C$,  number of available examples per class $b_1, \ldots, b_C$}
\Output{length $L$ sequence of examples, kernel matrix of size $L \times L$, number of clusters present}
\BlankLine
Initialize $O \gets [\; ]$\\
Pick $k \gets$ \Uniform{$1,$ \Min{$C, L$}}\\
Pick per cluster frequencies $n_1, \cdots, n_k$ with $1 \le n_i \le b_i$ and $\sum_{i=1}^{k} n_i = L$\\
\For{$i \gets 1$ \KwTo $k$}{
    Pick a class not yet chosen uniformly at random\\
    append $n_i$ uniform examples of chosen class to $O$\\
}
Let $A \gets$ true kernel matrix corresponding to $O$\\
\Return{$O, A, k$}
\end{algorithm}

\newpage
\section{More details on Omniglot results}\label{sec:appendix_omniglot_details}
\subsection{Details of experimental setup}\label{subsec:appendix_experiment_hyperparams}
The results discussed in Section~\ref{sec:experiments} and shown in this Appendix are produced with the following hyperparameters: the embedding component uses two Self Attention Blocks (SAB), each with four heads. The dimensionality of keys, queries, and values is set to $128$. The learning rate is set to $0.001$. We found that using larger batch sizes of up to 128 tends to improve training.

\subsection{Normalized mutual information per alphabet}\label{subsec:appendix_nmi_per_alphabet}

In Table~\ref{tab:omnitest} we show more details on Omniglot testing results, split out per alphabet. The averages at the bottom of the table are reported in the main body of this paper in Table~\ref{tab:ABC_results}.
\begin{table}[ht]
    \centering
    \caption{Average NMI scores for 1000 random instances, each of size 100, for each alphabet in the evaluation set. The number of clusters varies uniformly up to the maximum available for each alphabet, which is 47 for Malayalam. `Mul' refers to multiplicative attention, while `Add' means ABC with additive attention.}
    \begin{tabular}{|l||c|c|c|c|c|c|}
        \hline
        \multirow{2}{*}{Alphabet} & \multicolumn{2}{c|}{$k \in [5, 47]$  (unk.)} & \multicolumn{2}{c|}{$k \in [5, 47]$ (known)} & \multicolumn{2}{c|}{$k = 20$ (known)} \\
        \cline{2-7}
        & Mul & Add & Mul & Add & Mul & Add \\
        \hline\hline
        Angelic & 0.8944 & 0.8566 & 0.8977 & 0.8757 & 0.8593 & 0.8435 \\
        Atemayar\textunderscore Qel. & 0.8399 & 0.8003 & 0.8761 & 0.8570 & 0.8692 & 0.8315 \\
        Atlantean & 0.9182 & 0.8927 & 0.9272 & 0.9188 & 0.9104 & 0.8994 \\
        Aurek-Besh & 0.9371 & 0.9247 & 0.9444 & 0.9312 & 0.9367 & 0.9247 \\
        Avesta & 0.9011 & 0.8728 & 0.9067 & 0.8956 & 0.8939 & 0.8733 \\
        Ge\textunderscore ez & 0.8877 & 0.8833 & 0.8931 & 0.8943 & 0.8725 & 0.8864 \\
        Glagolitic & 0.9046 & 0.8366 & 0.9186 & 0.8965 & 0.9158 & 0.8943 \\
        Gurmukhi & 0.8685 & 0.7999 & 0.8949 & 0.8668 & 0.9018 & 0.8674 \\
        Kannada & 0.8120 & 0.6837 & 0.8545 & 0.8267 & 0.8648 & 0.8225 \\
        Keble & 0.8671 & 0.8195 & 0.8921 & 0.8623 & 0.9042 & 0.8291 \\
        Malayalam & 0.8810 & 0.8494 & 0.8963 & 0.8869 & 0.8909 & 0.8854 \\
        Manipuri & 0.9035 & 0.8637 & 0.9152 & 0.8948 & 0.9039 & 0.8918 \\
        Mongolian & 0.9200 & 0.8879 & 0.9277 & 0.9143 & 0.9176 & 0.9020 \\
        \small{Old\textunderscore Church...} & 0.9358 & 0.9336 & 0.9419 & 0.9425 & 0.9302 & 0.9372 \\
        Oriya & 0.8008 & 0.6734 & 0.8460 & 0.8019 & 0.8466 & 0.7912 \\
        Sylheti & 0.7725 & 0.6414 & 0.8220 & 0.7923 & 0.8151 & 0.7708 \\
        Syriac\textunderscore (Serto) & 0.8909 & 0.8381 & 0.8946 & 0.8762 & 0.8794 & 0.8535 \\
        Tengwar & 0.8758 & 0.8359 & 0.8872 & 0.8697 & 0.8571 & 0.8524 \\
        Tibetan & 0.8840 & 0.8694 & 0.8996 & 0.8961 & 0.8982 & 0.8935 \\
        ULOG & 0.7895 & 0.5621 & 0.8185 & 0.7656 & 0.8132 & 0.7544 \\
        \hline
        \textbf{mean} & 0.8742 & 0.8163 & 0.8927 & 0.8733 & 0.8840 & 0.8602 \\
        \hline
    \end{tabular}
    \label{tab:omnitest}
\end{table}

\end{document}